\documentclass{article}

\usepackage{times}
\usepackage{graphicx} 
\usepackage{subfigure}

\usepackage{natbib}

\usepackage{algorithm}
\usepackage{algorithmic}
\usepackage[algo2e,linesnumbered,lined,boxed,vlined]{algorithm2e}

\usepackage{amsfonts}
\usepackage{amssymb}

\usepackage{hyperref}

\usepackage[accepted]{icml2016}

\usepackage{comment}
\usepackage{amsthm}
\usepackage{mathtools}
\usepackage{tabularx}
\usepackage{multirow}

\newtheorem{theorem}{Theorem}
\newtheorem{lemma}{Lemma}

\newtheorem{definition}{Definition}

\def\b{\ensuremath\boldsymbol}

\usepackage{lastpage}
\usepackage{fancyhdr}
\usepackage{url}
\icmltitlerunning{}

\begin{document}

\AddToShipoutPictureBG*{%
  \AtPageUpperLeft{%
    \setlength\unitlength{1in}%
    \hspace*{\dimexpr0.5\paperwidth\relax}
    \makebox(0,-0.75)[c]{\normalsize {\color{black} To appear as a part of an upcoming textbook on dimensionality reduction and manifold learning.}}
    }}

\twocolumn[
\icmltitle{Unified Framework for Spectral Dimensionality Reduction, Maximum Variance Unfolding, and Kernel Learning By Semidefinite Programming: Tutorial and Survey}

\icmlauthor{Benyamin Ghojogh}{bghojogh@uwaterloo.ca}
\icmladdress{Department of Electrical and Computer Engineering, 
\\Machine Learning Laboratory, University of Waterloo, Waterloo, ON, Canada}
\icmlauthor{Ali Ghodsi}{ali.ghodsi@uwaterloo.ca}
\icmladdress{Department of Statistics and Actuarial Science \& David R. Cheriton School of Computer Science, 
\\Data Analytics Laboratory, University of Waterloo, Waterloo, ON, Canada}
\icmlauthor{Fakhri Karray}{karray@uwaterloo.ca}
\icmladdress{Department of Electrical and Computer Engineering, 
\\Centre for Pattern Analysis and Machine Intelligence, University of Waterloo, Waterloo, ON, Canada}
\icmlauthor{Mark Crowley}{mcrowley@uwaterloo.ca}
\icmladdress{Department of Electrical and Computer Engineering, 
\\Machine Learning Laboratory, University of Waterloo, Waterloo, ON, Canada}

\icmlkeywords{Tutorial}

\vskip 0.3in
]

\begin{abstract}
This is a tutorial and survey paper on unification of spectral dimensionality reduction methods, kernel learning by Semidefinite Programming (SDP), Maximum Variance Unfolding (MVU) or Semidefinite Embedding (SDE), and its variants. We first explain how the spectral dimensionality reduction methods can be unified as kernel Principal Component Analysis (PCA) with different kernels. This unification can be interpreted as eigenfunction learning or representation of kernel in terms of distance matrix. Then, since the spectral methods are unified as kernel PCA, we say let us learn the best kernel for unfolding the manifold of data to its maximum variance. We first briefly introduce kernel learning by SDP for the transduction task. Then, we explain MVU in detail. Various versions of supervised MVU using nearest neighbors graph, by class-wise unfolding, by Fisher criterion, and by colored MVU are explained. We also explain out-of-sample extension of MVU using eigenfunctions and kernel mapping. Finally, we introduce other variants of MVU including action respecting embedding, relaxed MVU, and landmark MVU for big data.
\end{abstract}

\section{Introduction}

Dimensionality reduction algorithms can be divided into three categories which are spectral, probabilistic, and neural network-based methods \cite{ghojogh2021data}. Spectral dimensionality reduction methods deal with geometry of data and sometimes reduce to an eigenvalue or generalized eigenvalue problem \cite{ghojogh2019eigenvalue}. Various spectral methods have been proposed during decades. Some of the most well-known spectral methods are Principal Component Analysis (PCA) \cite{ghojogh2019unsupervised}, Multidimensional Scaling (MDS), Isomap \cite{ghojogh2020multidimensional}, spectral clustering, Laplacian eigenmap, diffusion map \cite{ghojogh2021laplacian}, and Locally Linear Embedding (LLE) \cite{ghojogh2020locally}.
After development of many of these methods, it was found out that the spectral methods are all learning eigenfunctions and can be reduced to kernel PCA with different kernels \cite{bengio2003learning,bengio2003spectral,bengio2003out,bengio2004learning,bengio2006spectral}. Moreover, from the formulation of MDS, we know that we can write kernel in terms of the distance matrix \cite{ghojogh2020multidimensional}. It was shown in \cite{ham2004kernel} that the spectral methods can be seen as kernel PCA with different kernels where kernels are constructed from various distance matrices. Hence, the spectral methods can be unified as kernel PCA theoretically. 
This kernel-based unified framework for the spectral dimensionality reduction encouraged the researchers to obtain the best kernel matrix for every specific dataset.

Around the time of discovery of this unification, it was found out in \cite{lanckriet2004learning} that the kernel matrix can be learned using Semidefinite Programming (SDP) \cite{vandenberghe1996semidefinite}. 
This kernel learning was proposed for the goal of transduction, i.e., learning the labels of an unlabeled part of data. 
The fact that the kernel can be learned by SDP inspired researchers to use SDP for learning the best kernel for dimensionality reduction and manifold unfolding. They said now that we knew the spectral methods can be seen as kernel PCA with different kernels, let us learn the best kernel for manifold unfolding to its maximum variance. Hence, Semidefinite Embedding (SDE) \cite{weinberger2004learning,weinberger2004unsupervised,weinberger2005nonlinear,weinberger2006unsupervised} was proposed which was later renamed to Maximum Variance Unfolding (MVU) \cite{weinberger2006introduction}.
MVU unfolds manifold to its maximum variance in its intrinsic dimensionality. 
For understanding the intuition of intrinsic dimensionality, see Fig. \ref{figure_intrinsic_dimensionality}. The task of manifold unfolding to its maximum variance is depicted by an example in Fig. \ref{figure_swill_roll_unfold}.
After proposal of MVU, various versions of MVU were developed such as supervised MVU \cite{zhang2005supervised,liu2005supervised,song2007colored,wei2016developments}, landmark MVU \cite{weinberger2005nonlinear}, action respecting embedding \cite{bowling2005action}, out-of-sample extensions \cite{chin2008out}, relaxed MVU \cite{hou2008relaxed}, etc. 

\begin{figure}[!t]
\centering
\includegraphics[width=3.2in]{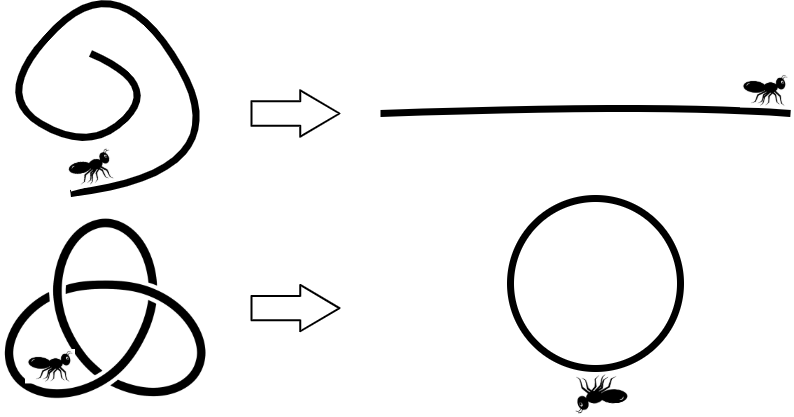}
\caption{The intrinsic dimensionality for Swiss roll (above) and trefoil knot (below). Suppose a small ant, which cannot see the whole manifold together and can only see its front, traverses the manifold once completely. The dimensionality that the ant feels by this traversing is the intrinsic dimensionality. For example, the intrinsic dimensionalities of 2D Swiss roll and 3D trefoil knot are one and two, respectively, because the ant reaches an end of line in the former (so it is like a 1D line) and the ant reaches to its starting point in the latter (so it is like a 2D circle).}
\label{figure_intrinsic_dimensionality}
\end{figure}

\begin{figure*}[!t]
\centering
\includegraphics[width=6.5in]{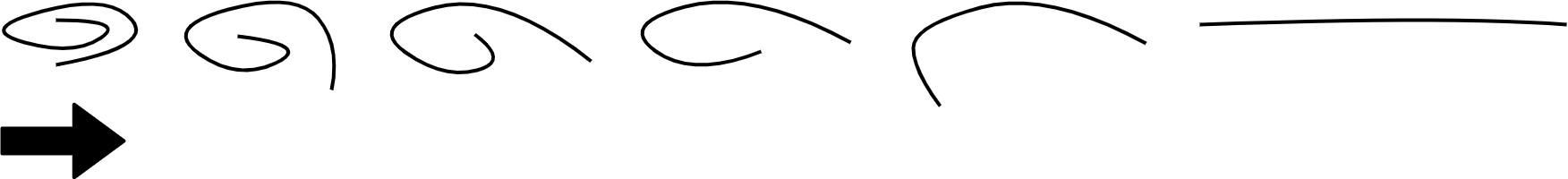}
\caption{Iterative unfolding of a nonlinear Swiss roll manifold using maximum variance unfolding.}
\label{figure_swill_roll_unfold}
\end{figure*}

The remainder of this paper is organized as follows. We explain unification of spectral dimensionality reduction methods as kernel PCA in Section \ref{section_unified_framework}. Required background on SDP is introduced in Section \ref{section_semidefinite_programming_background}. We introduce kernel learning by SDP for the transduction task in Section \ref{section_kernel_learning_transduction}. MVU or SDE is explained in detail in Section \ref{section_MVU}. We explain supervised MVU using nearest neighbors graph, by class-wise unfolding, by Fisher criterion, and by colored MVU in Section \ref{section_supervised_MVU}. Out-of-sample extensions of MVU using eigenfunctions and kernel mapping are explained in Section \ref{section_MVU_outofsample}. Some other variants of MVU including action respecting embedding, relaxed MVU, and landmark MVU for big data are introduced in Section \ref{section_MVU_other_variants}. Finally, Section \ref{section_conclusion} concludes the paper. 

\section*{Required Background for the Reader}

This paper assumes that the reader has general knowledge of calculus, linear algebra, and basics of optimization. 

\section{Unified Framework for Spectral Methods}\label{section_unified_framework}

After development of many spectral dimensionality reduction methods, it was found out by researchers that these methods all can be unified as kernel Principal Component Analysis (PCA) \cite{ghojogh2019unsupervised} with different kernels. This is mostly because the spectral methods reduce to an eigenvalue or a generalized eigenvalue problem \cite{ghojogh2019eigenvalue}. 
This unification can be analyzed from different perspectives including eigenfunction learning, having kernels in kernel PCA, and generalized embedding. In the following, we explain these points of view. 

\subsection{Learning Eigenfunctions}

After some initial work was done in \cite{williams2000effect} on using eigenfunctions in machine learning, Bengio et. al. developed eigenfunctions learning and showed that the spectral methods actually learn eigenfunctions \cite{bengio2003learning,bengio2003spectral,bengio2003out,bengio2004learning,bengio2006spectral}. 
In the following, we repeat some contents from our other tutorial \cite{ghojogh2021reproducing} for the sake of explanation. Refer to \cite{ghojogh2021reproducing} for more details and proofs about eigenfunctions and embedding using eigenfunctions.

First, consider The eigenvalue problem for the kernel matrix \cite{ghojogh2019eigenvalue}: 
\begin{align}\label{equation_eigenvalue_problem_kernel_operator}
\b{K} \b{v}_k = \delta_k \b{v}_k, \quad \forall k \in \{1, \dots, n\},
\end{align}
where $\b{v}_k$ is the $k$-th eigenvector and $\delta_k$ is the corresponding eigenvalue. 

\begin{definition}[Eigenfunction \cite{ghojogh2021reproducing}]\label{definition_eigenfunction}
Consider a linear operator $O$ which can be applied on a function $f$. If applying this operator on the function results in a multiplication of function to a constant:
\begin{align}\label{equation_eigenfunction}
O f = \lambda f,
\end{align}
then the function $f$ is an eigenfunction for the operator $O$ and the constant $\lambda$ is the corresponding eigenvalue.
\end{definition}
Now, in the feature space $\mathcal{H}$, consider a kernel operator $K_p$ as \cite{williams2000effect}, {\citep[Section 3]{bengio2003out}}:
\begin{align}\label{equation_K_operator_integral}
(K_p f)(\b{x}) := \int k(\b{x},\b{y})\, f(\b{y})\, p(\b{y})\, d\b{y},
\end{align}
where $f \in \mathcal{H}$ and the density function $p(\b{y})$ can be approximated empirically. 
The eigenfunction problem for this kernel operator is \cite{ghojogh2021reproducing}:
\begin{align}
&(K_p f_k)(\b{x}) = \lambda_k f_k(\b{x}), \quad \forall k \in \{1, \dots, n\},
\end{align}
where $f_k(.)$ is the $k$-th eigenfunction and $\lambda_k$ is the corresponding eigenvalue.

\begin{lemma}[Relation of Eigenfunctions and Eigenvectors for Kernel {\citep[Proposition 1]{bengio2003out}}, {\citep[Theorem 1]{bengio2003spectral}}]\label{lemma_relation_eigenfunctions_eigenvectors_for_kernel}
Consider a training dataset $\{\b{x}_i \in \mathbb{R}^d\}_{i=1}^n$ and the eigenvalue problem (\ref{equation_eigenvalue_problem_kernel_operator}) where $\b{v}_k \in \mathbb{R}^n$ and $\delta_k$ are the $k$-th eigenvector and eigenvalue of matrix $\b{K} \in \mathbb{R}^{n \times n}$.
If $v_{ki}$ is the $i$-th element of vector $\b{v}_k$, the eigenfunction for the point $\b{x}$ and the $i$-th training point $\b{x}_i$ are:
\begin{align}
&f_k(\b{x}) = \frac{\sqrt{n}}{\delta_k} \sum_{i=1}^n v_{ki}\, \breve{k}(\b{x}_i, \b{x}), \label{equation_relation_eigenfunction_eigenvector_x} \\
&f_k(\b{x}_i) = \sqrt{n}\, v_{ki}, \label{equation_relation_eigenfunction_eigenvector_x_i}
\end{align}
respectively, where $\breve{k}(\b{x}_i, \b{x})$ is the centered kernel. 
If $\b{x}$ is a training point, $\breve{k}(\b{x}_i, \b{x})$ is the centered kernel over training data and if $\b{x}$ is an out-of-sample point, then $\breve{k}(\b{x}_i, \b{x}) = \breve{k}_t(\b{x}_i, \b{x})$ is between training set and the out-of-sample point (n.b. see \cite{ghojogh2021reproducing} for information on kernel centering). 
\end{lemma}

\begin{theorem}[Embedding from Eigenfunctions of Kernel Operator {\citep[Proposition 1]{bengio2003out}}, {\citep[Section 4]{bengio2003spectral}}]\label{theorem_embedding_from_eigenfunctions}
Consider a dimensionality reduction algorithm which embeds data into a low-dimensional embedding space.
Let the embedding of the point $\b{x}$ be $\mathbb{R}^p \ni \b{y}(\b{x}) = [y_1(\b{x}), \dots, y_p(\b{x})]^\top$ where $p \leq n$. The $k$-th dimension of this embedding is:
\begin{align}\label{equation_embedding_eigenfunction}
y_k(\b{x}) &= \sqrt{\delta_k}\, \frac{f_k(\b{x})}{\sqrt{n}} = \frac{1}{\sqrt{\delta_k}} \sum_{i=1}^n v_{ki}\, \breve{k}(\b{x}_i, \b{x}),
\end{align}
where $\breve{k}(\b{x}_i, \b{x})$ is the centered training or out-of-sample kernel depending on whether $\b{x}$ is a training or an out-of-sample point.
\end{theorem}

The Theorem \ref{theorem_embedding_from_eigenfunctions} has been widely used for out-of-sample (test data) embedding in many spectral dimensionality reduction algorithms \cite{bengio2003out}. This theorem also explains why spectral methods can all be seen as kernel PCA for learning eigenfunctions. 
In the following, we justify this insight into unification of the spectral methods.

\subsection{Unified Framework as Kernel PCA}

Many spectral dimensionality reduction methods can be reduced to kernel PCA where the eigenvectors of the kernel matrix or eigenfunctions of the kernel operator are used for embedding as was stated in Lemma \ref{lemma_relation_eigenfunctions_eigenvectors_for_kernel} and Theorem \ref{theorem_embedding_from_eigenfunctions}.
This unification was analyzed in the following two categories of papers:
\begin{itemize}
\item papers for eigenfunction learning: \cite{bengio2003learning,bengio2003spectral,bengio2003out,bengio2004learning,bengio2006spectral}
\item papers for unification by kernels constructed from distance matrices: \cite{ham2004kernel} and {\citep[Table 2.1]{strange2014open}} and \cite{ghojogh2019feature}
\end{itemize}
In the following, we explain this unification of spectral methods as kernel PCA using both approaches together. 

\hfill\break
\textbf{-- Principal Component analysis (PCA):}

As kernel PCA makes use of kernel trick for kernelization, it is equivalent to PCA when a linear kernel is used \cite{ghojogh2021reproducing}. Hence, PCA \cite{ghojogh2019unsupervised} is equivalent to kernel PCA by the linear kernel:
\begin{align}
\b{K} = \b{X}^\top \b{X}.
\end{align}

\textbf{-- Multidimensional Scaling (MDS):}

According to our derivations in \cite{ghojogh2020multidimensional} or \cite{ghojogh2021reproducing}, MDS \cite{cox2008multidimensional} is equivalent to kernel PCA with the kernel:
\begin{align}
\b{K} = -\frac{1}{2} \b{H} \b{D} \b{H},
\end{align}
where $\b{D}$ is the squared Euclidean distance matrix and $\b{H} := \b{I} - (1/n) \b{1} \b{1}^\top$ is the centering matrix.

\hfill\break
\textbf{-- Spectral Clustering:}

Let $\b{W} \in \mathbb{R}^{n \times n}$ be the adjacency matrix of points \cite{ghojogh2021laplacian}. Spectral clustering \cite{weiss1999segmentation,ng2001spectral,ghojogh2021laplacian} uses the normalized adjacency matrix as its kernel matrix. Suppose $\b{D}_{i,i} := \sum_{j=1}^n \b{W}_{ij}$ is the $(i,i)$-th element of the diagonal degree matrix of adjacency matrix. 
Hence, the $(i,j)$-th element of kernel is \cite{weiss1999segmentation,bengio2003out}:
\begin{align}\label{equation_kernel_spectral_clustering_normalized}
\b{K}_{ij} = \frac{\b{W}_{ij}}{\sqrt{\b{D}_{ii}\, \b{D}_{jj}}},
\end{align}
where $\b{K}_{ij} = k(\b{x}_i, \b{x}_j)$ for the kernel function $k(.,.)$.
This relation of kernel matrix with the adjacency matrix makes sense because kernel is a notion of similarity \cite{ghojogh2021reproducing}.

\hfill\break
\textbf{-- Laplacian Eigenmap:}

The solution of Laplacian eigenmap \cite{belkin2001laplacian,ghojogh2021laplacian} is the generalized eigenvalue problem $(\b{L}, \b{D})$ \cite{ghojogh2019eigenvalue} where $\b{L} := \b{D} - \b{W}$ is the Laplacian of the adjacency matrix $\b{W}$. According to {\citep[Normalization Lemma 1]{weiss1999segmentation}} , the eigenvectors of the generalized eigenvalue problem $(\b{L}, \b{D})$ are equivalent to the eigenvectors of the normalized kernel in Eq. (\ref{equation_kernel_spectral_clustering_normalized}).

In addition to the above analysis \cite{bengio2003out}, there exists another analysis for Laplacian eigenmap stated in \cite{ham2004kernel}. Consider the Laplacian of graph of data. It can, for example, be the Laplacian of RBF adjacency matrix. It was mentioned above that the solution of Laplacian eigenmap is the generalized eigenvalue problem $(\b{L}, \b{D})$. However, there exists another form of optimization for Laplacian eigenmap whose solution is the eigenvalue problem for $\b{L}$ \cite{ghojogh2021laplacian}. As the optimization of Laplacian eigenmap is minimization, its eigenvectors are sorted from the smallest to largest eigenvalues \cite{ghojogh2019eigenvalue}. However, the optimization in PCA and kernel PCA is maximization and the eigenvectors are sorted from the largest to smallest eigenvalues. 
Hence, to have equivalency with kernel PCA, we should modify minimization to maximization. One way to do this is to replace $\b{L}$ with \cite{ham2004kernel}:
\begin{align}\label{equation_kernel_inverse_Laplacian}
\mathbb{R}^{n \times n} \ni \b{K} = \b{L}^\dagger,
\end{align}
where $^\dagger$ denotes the pseudo-inverse of matrix. This replacement changes minimization to maximization because the Laplacian matrix is positive semidefinite \cite{ghojogh2021laplacian} and eigenvalues of inverse of a positive semidefinite matrix are equivalent to the reciprocal of eigenvalues that matrix while the eigenvectors remain the same. Hence, the order of eigenvalues and eigenvectors become reversed.
It is also important that the pseudo-inverse of Laplacian used as the kernel should be double-centered because kernel is double-centered in kernel PCA \cite{ghojogh2019unsupervised}. It is interesting that Eq. (\ref{equation_kernel_inverse_Laplacian}) is already double-centered \cite{ham2004kernel}:
\begin{align*}
\b{L} \b{1} = \b{L}^\dagger \b{1} = \b{0} \implies \b{H} \b{L}^\dagger \b{H} = \b{L}^\dagger,
\end{align*}
where $\b{H} := \b{I} - (1/n) \b{1} \b{1}^\top$ is the centering matrix. Note that $\b{L} \b{1} = \b{L}^\dagger \b{1} = \b{0}$ holds because the row summation of Laplacian matrix is zero \cite{ghojogh2021laplacian}.

\hfill\break
\textbf{-- Isomap:}

The kernel in Isomap is \cite{tenenbaum2000global,ghojogh2020multidimensional}:
\begin{align}
\b{K} = -\frac{1}{2} \b{H} \b{D}^{(g)} \b{H},
\end{align}
where $\b{D}^{(g)}$ is the geodesic distance matrix whose elements are the (squared) approximation of geodesic distances using piece-wise Euclidean distances \cite{ghojogh2020multidimensional}. Hence, it can be seen as kernel PCA with the above-mentioned kernel \cite{ham2004kernel}. 

\hfill\break
\textbf{-- Locally Linear Embedding (LLE):}



The solution of LLE is the eigenvectors of $\b{M} := (\b{I} - \b{W})^\top (\b{I} - \b{W})$ where $\b{I}$ is the identity matrix and $\b{W}$ is the weight matrix obtained from the linear reconstruction step in LLE \cite{roweis2000nonlinear,ghojogh2020locally}. As the optimization of LLE is minimization, its eigenvalues are sorted from smallest to largest \cite{ghojogh2019eigenvalue}. However, PCA and kernel PCA have maximization in their formulation. Hence, to have equivalency with kernel PCA, we should modify minimization to maximization. One way to do this is to replace $\b{M}$ with \cite{scholkopf2002learning,bengio2003learning}:
\begin{align}\label{equation_kernel_I_minus_M_LLE}
\mathbb{R}^{n \times n} \ni \b{K} = \mu \b{I} - \b{M},
\end{align}
where $\mu > 0$. 
It is suggested in \cite{ham2004kernel} to set $\mu$ to the largest eigenvalue $\lambda_\text{max}$ of $\b{M}$ so that the kernel $\b{K}$ becomes positive definite. 
This replacement of $\b{M}$ with the provided kernel changes minimization to maximization because:
\begin{align*}
\arg\max(\b{K}) \overset{(\ref{equation_kernel_I_minus_M_LLE})}{=} \arg\max(-\b{M}) = \arg\min(\b{M}).
\end{align*}
Another possible replacement exists and that is replacing $\b{M}$ with its pseudo-inverse {\citep[Section 2.1]{alipanahi2011guided}}:
\begin{align}\label{equation_kernel_inverse_M}
\mathbb{R}^{n \times n} \ni \b{K} = \b{M}^\dagger,
\end{align}
whose justification is similar to justification of pseudo-inverse of Laplacian for Laplacian eigenmap. 

\hfill\break
\textbf{-- Diffusion Map:}

Diffusion map \cite{coifman2006diffusion} is a Laplacian-based method \cite{ghojogh2021laplacian} and can be seen as a special case of kernel PCA with a specific kernel \cite{ham2004kernel}, {\citep[Chapter 2]{strange2014open}}.
Let $\b{M}$ denote the random-walk graph Laplacian normalization of Laplacian, i.e., $\b{M} := (\b{D}^{(\alpha)})^{-1} \b{L}^{(\alpha)}$ where $\alpha$ is a parameter. At time $t$, we can say that $\b{M}^t$ is the probability of going from $\b{x}_i$ to $\b{x}_j$. The solution of diffusion map at time $t$ is the eigenvectors of $\b{M}^t$ sorted from largest to smallest eigenvalues \cite{ghojogh2021laplacian}. Hence, we can consider the following kernel for equivalency of diffusion map and kernel PCA:
\begin{align}
\mathbb{R}^{n \times n} \ni \b{K} = \b{M}^t,
\end{align}
because kernel PCA also has maximization in its formulation. 

\subsection{Summary of Kernels in Spectral Methods}

As we saw, kernels in some algorithms such as PCA and MDS are almost closed-form and can be computed almost fast. However, some methods such as Isomap require a piece of code to calculate their kernels. 
We can have a summary of kernels for unification of spectral dimensionality reduction methods as kernel PCA. A similar summary exists in {\citep[Table 2.1]{strange2014open}}.
We summarize the kernels in Table \ref{table_unification_spectral_kernels}.

\begin{table}[!t]
\caption{The summary of kernels for unification of spectral methods as kernel PCA \hfill\break}
\label{table_unification_spectral_kernels}
\renewcommand{\arraystretch}{1.3}  
\centering
\scalebox{1}{    
\begin{tabular}{l | c}
\hline
\hline
Method & Kernel \\
\hline
PCA & $\b{X}^\top \b{X}$ \\
\hline
MDS & $-\frac{1}{2} \b{H} \b{D} \b{H}$ \\
\hline
Spectral clustering & $\frac{\b{W}_{ij}}{\sqrt{\b{D}_{ii}\, \b{D}_{jj}}}$ \\
\hline
Laplacian eigenmap & $\b{L}^\dagger$ \,or\, $\frac{\b{W}_{ij}}{\sqrt{\b{D}_{ii}\, \b{D}_{jj}}}$ \\
\hline
Isomap & $-\frac{1}{2} \b{H} \b{D}^{(g)} \b{H}$ \\
\hline
LLE & $\b{M}^\dagger$ \,or\, $\lambda_\text{max} \b{I} - \b{M}$ \\
\hline
Diffusion map & $\b{M}^t$ \\
\hline
\hline
\end{tabular}%
}
\end{table}

\subsection{Generalized Embedding}

It was demonstrated above that many of the spectral dimensionality reduction methods belong to a unified framework. Therefore, there can be generalized embedding methods which generalize the spectral methods to broader algorithms.
Graph Embedding (GE) \cite{yan2005graph,yan2006graph,ghojogh2021laplacian} showed that many spectral methods, including Laplacian eigenmap, locality preserving projection, PCA and kernel PCA, Fisher Discriminant Analysis (FDA) and kernel FDA, MDS, Isomap, and LLE, are special cases of a model named graph embedding. Another generalized subspace learning method, named Roweis Discriminant Analysis (RDA), also generalized PCA, supervised PCA, and FDA as unified methods. These generalized methods justify from another perspective that why spectral methods can be unified. 

\section{Background on Semidefinite Programming}\label{section_semidefinite_programming_background}

Here, we review some optimization background required for solving SDP used in MVU and kernel learning. 

\subsection{Unconstrained Optimization}

Consider the following optimization problem:
\begin{align}
\underset{\b{x}}{\text{minimize}}\quad f(\b{x}).
\end{align}
where $f(.)$ is a convex function.
Iterative optimization can be first-order or second-order. Iterative optimization updates solution iteratively:
\begin{align}
\b{x} \gets \b{x} + \Delta\b{x},
\end{align}
until $\Delta\b{x}$ becomes very small which is the convergence of optimization.
In the first-order optimization, the step of updating is $\Delta\b{x} := -\nabla f(\b{x})$ where $\nabla f(.)$ denotes the gradient of function.
Near the optimal point $\b{x}^*$, gradient is very small so the second-order Taylor series expansion of function becomes:
\begin{align*}
f(\b{x}) \approx &\,f(\b{x}^*) + \underbrace{\nabla f(\b{x}^*)^\top}_{\approx\, 0} (\b{x} - \b{x}^*) \\
&+ \frac{1}{2} (\b{x} - \b{x}^*)^\top \nabla^2 f(\b{x}^*) (\b{x} - \b{x}^*) \\
&\approx f(\b{x}^*) + \frac{1}{2} (\b{x} - \b{x}^*)^\top \nabla^2 f(\b{x}^*) (\b{x} - \b{x}^*),
\end{align*}
where $\nabla^2 f(\b{x})$ denotes the second-order gradient (Hessian) of function.
This shows that the function is almost quadratic near the optimal point. Following this intuition, Newton's method uses Hessian $\nabla^2 f(\b{x})$ in its updating step:
\begin{align}
\Delta\b{x} := - \nabla^2 f(\b{x})^{-1} \nabla f(\b{x}).
\end{align}

\subsection{Equality Constrained Optimization}

The optimization problem may have equality constraint:
\begin{equation}\label{equation_optimization_problem_equality_constraint}
\begin{aligned}
& \underset{\b{x}}{\text{minimize}}
& & f(\b{x}) \\
& \text{subject to}
& & \b{A} \b{x} = \b{b}.
\end{aligned}
\end{equation}
After a step of update by $\Delta \b{x} = \b{u}$, this optimization becomes:
\begin{equation}
\begin{aligned}
& \underset{\b{x}}{\text{minimize}}
& & f(\b{x} + \b{u}) \\
& \text{subject to}
& & \b{A} (\b{x} + \b{u}) = \b{b}.
\end{aligned}
\end{equation}

The Lagrangian of this optimization problem is \cite{boyd2004convex}:
\begin{align*}
\mathcal{L} = f(\b{x} + \b{u}) + \b{\nu}^\top (\b{A} (\b{x} + \b{u}) - \b{b}),
\end{align*}
where $\b{\nu}$ is the dual variable. 
The second-order Taylor series expansion of function $f(\b{x} + \b{u})$ is:
\begin{align}
f(\b{x} + \b{u}) \approx f(\b{x}) + \nabla f(\b{x})^\top \b{u} + \frac{1}{2} \b{u}^\top \nabla^2 f(\b{x}^*)\, \b{u}.
\end{align}
Substituting this into the Lagrangian gives:
\begin{align*}
\mathcal{L} = f(\b{x}) &+ \nabla f(\b{x})^\top \b{u} + \frac{1}{2} \b{u}^\top \nabla^2 f(\b{x}^*)\, \b{u} \\
&+ \b{\nu}^\top (\b{A} (\b{x} + \b{u}) - \b{b}).
\end{align*}
According to Karush–Kuhn–Tucker (KKT) conditions, the primal and dual residuals must be zero:
\begin{align}
& \nabla_{\b{x}} \mathcal{L} = \nabla f(\b{x}) + \nabla^2 f(\b{x})^\top \b{u} + \b{u}^\top \underbrace{\nabla^3 f(\b{x}^*)}_{\approx\, \b{0}}\, \b{u} \nonumber \\
&+ \b{A}^\top \b{\nu} \overset{\text{set}}{=} \b{0} \implies \nabla^2 f(\b{x})^\top \b{u} + \b{A}^\top \b{\nu} = - \nabla f(\b{x}), \label{equation_Newton_method_derivative_x_zero} \\
& \nabla_{\b{\nu}} \mathcal{L} = \b{A} (\b{x} + \b{u}) - \b{b} \overset{(a)}{=} \b{A} \b{u} \overset{\text{set}}{=} \b{0}, \label{equation_Newton_method_derivative_nu_zero}
\end{align}
where we have $\nabla^3 f(\b{x}^*) \approx 0$ because the gradient of function at the optimal point vanishes and $(a)$ is because of the constraint $\b{Ax} - \b{b} = \b{0}$ in problem (\ref{equation_optimization_problem_equality_constraint}).
Eqs. (\ref{equation_Newton_method_derivative_x_zero}) and (\ref{equation_Newton_method_derivative_nu_zero}) can be written as a system of equations:
\begin{align}
\begin{bmatrix}
\nabla^2 f(\b{x})^\top & \b{A}^\top \\
\b{A} & \b{0}
\end{bmatrix}
\begin{bmatrix}
\b{u} \\
\b{\nu}
\end{bmatrix} 
=
\begin{bmatrix}
-\nabla f(\b{x}) \\
\b{0}
\end{bmatrix}.
\end{align}
Solving this system of equations gives the desired step $\b{u}$ for updating the solution at the iteration.

\subsection{Inequality Constrained Optimization}

The optimization problem may have equality constraint:
\begin{equation}
\begin{aligned}
& \underset{\b{x}}{\text{minimize}}
& & f(\b{x}) \\
& \text{subject to}
& & f_i(\b{x}) \leq 0, \quad i \in \{1, \dots, m\}, \\
& & & \b{A} \b{x} = \b{b}.
\end{aligned}
\end{equation}

Barrier methods, also known as interior-point methods, convert inequality constrained problems to equality constrained or unconstrained problems. Ideally, we can do this conversion using the indicator function $\mathbb{I}(.)$ which is zero if its input condition is satisfied and is infinity otherwise (n.b. the indicator function in optimization literature is not like the indicator in data science which is one if its input condition is satisfied and is zero otherwise). The problem is converted to:
\begin{equation}
\begin{aligned}
& \underset{\b{x}}{\text{minimize}}
& & f(\b{x}) + \sum_{i=1}^m \mathbb{I}(f_i(\b{x}) \leq 0) \\
& \text{subject to}
& & \b{A} \b{x} = \b{b}.
\end{aligned}
\end{equation}
The indicator function is not differentiable because it is not smooth. Hence, we can approximate it with differentiable functions such as logarithm. Logarithmic barrier approximates the indicator function by:
\begin{align}
\mathbb{I}(f_i(\b{x}) \leq 0) \approx -\frac{1}{t} \log(-f_i(\b{x})),
\end{align}
where the approximation becomes more accurate by $t \rightarrow \infty$.
It changes the problem to:
\begin{equation}
\begin{aligned}
& \underset{\b{x}}{\text{minimize}}
& & f(\b{x}) - \frac{1}{t} \sum_{i=1}^m \log(-f_i(\b{x})) \\
& \text{subject to}
& & \b{A} \b{x} = \b{b}.
\end{aligned}
\end{equation}
This optimization problem is an equality constrained optimization problem which we already explained how to solve. 

\subsection{Semidefinite Programming}

An optimization problem in the following form is a Semidefinite Programming (SDP) \cite{vandenberghe1996semidefinite}:
\begin{equation}
\begin{aligned}
& \underset{\b{X}}{\text{minimize}}
& & \textbf{tr}(\b{C}^\top \b{X}) \\
& \text{subject to}
& & \textbf{tr}(\b{A}_i^\top \b{X}) = \b{b}_i, \quad i \in \{1, \dots, m_1\}, \\
& & & \textbf{tr}(\b{D}_i^\top \b{X}) \leq \b{e}_i, \quad i \in \{1, \dots, m_2\}, \\
& & & \b{X} \succeq \b{0},
\end{aligned}
\end{equation}
where $\b{C}$, $\b{A}_i$'s, and $\b{D}_i$'s are constant matrices and $\b{b}_i$'s and $\b{e}_i$'s are constant vectors. Note that the trace terms may be written in summation forms. 
The interior-point method, or the barrier method, introduced before, can be used for solving SDP \cite{nesterov1994interior,boyd2004convex}. Optimization toolboxes such as CVX \cite{grant2008cvx} often use interior-point method for solving optimization problems such as SDP. Note that this method is iterative and SDP solving usually is time consuming especially for large matrices. 
Also note that SDP is a convex optimization problem so it has only one local optimum which is the global optimum \cite{boyd2004convex}.

\section{Kernel Learning for Transduction}\label{section_kernel_learning_transduction}

Kernel learning by semidefinite programming was initially proposed in \cite{lanckriet2004learning} for the goal of transduction.
Transduction is a task in which the labeling of a not-completely-labeled dataset gets complete. In other words, using the labeled part of data, the embedding for the unlabeled part of data is also calculated. 
It seems that the paper \cite{lanckriet2004learning} has inspired the authors of MVU or SDE \cite{weinberger2004learning} to use semidefinite programming in kernel learning for the task of dimensionality reduction. 
As kernel learning for transduction is not completely related to dimensionality reduction, we briefly introduce it and do not enter its details. The reader can refer to \cite{lanckriet2004learning} or its summary in \cite{karimi2017summary} for more information about it. 

Consider a dataset which is partially labeled. The labeled and unlabeled sets of data can be the training and test sets, respectively. Let the training and test sets be denoted by $\{(\b{x}_i, y_i)\}_{i=1}^{n_\text{tr}}$ and $\{\b{x}_i\}_{i=n_\text{tr}+1}^{n_\text{te}}$, respectively, where $n_\text{tr}$ is the number of labeled training points and $n_\text{te}$ is the number of unlabeled test points and $n := n_\text{tr} + n_\text{te}$. 
The kernel matrix has the following sub-matrices:
\begin{align}
\mathbb{R}^{n \times n} \ni \b{K} = 
\begin{bmatrix}
\b{K}_{\text{tr}, \text{tr}} & \b{K}_{\text{tr}, \text{te}} \\
\b{K}_{\text{te}, \text{tr}} & \b{K}_{\text{te}, \text{te}},
\end{bmatrix},
\end{align}
where $\b{K}_{\text{tr}, \text{tr}} \in \mathbb{R}^{n_\text{tr} \times n_\text{tr}}$, $\b{K}_{\text{tr}, \text{te}} \in \mathbb{R}^{n_\text{tr} \times n_\text{te}}$, $\b{K}_{\text{te}, \text{tr}} \in \mathbb{R}^{n_\text{te} \times n_\text{tr}}$, and $\b{K}_{\text{te}, \text{te}} \in \mathbb{R}^{n_\text{te} \times n_\text{te}}$.
We will see later that MVU \cite{weinberger2004learning} learns the best kernel matrix for manifold unfolding. However, in contrast to kernel learning in MVU which learns the kernel matrix by optimization, the kernel learning proposed in \cite{lanckriet2004learning} finds the best kernel among a set of available kernels. Let $\mathcal{K} := \{\b{K}_1, \dots, \b{K}_m\}$ be the set of kernels whose elements can be different kernels such as linear kernel, RBF kernel, Laplacian kernel, etc \cite{ghojogh2021reproducing}. Kernel learning in this section learns the best kernel as one of the kernels in the set. 

We can learn the kernel among the sets of kernels by the following optimization problem {\citep[Theorem 16]{lanckriet2004learning}}:
\begin{equation}
\begin{aligned}
& \underset{\b{K}, \b{\nu}, \b{\delta}, \lambda, t}{\text{minimize}}
& & t \\
& \text{subject to}
& & \textbf{tr}(\b{K}) = c_1, \\
& & & \b{K} \in \mathcal{K}, \\
& & & \b{\nu} \geq \b{0}, \\
& & & \b{\delta} \geq \b{0}, \\
& & & 
\begin{bmatrix}
G(\b{K}_\text{tr}) + \tau \b{I} & \b{1} + \b{\nu} - \b{\delta} + \lambda \b{y}\\
(\b{1} + \b{\nu} - \b{\delta} + \lambda \b{y})^\top & t - 2\, c_2\, \b{\delta}^\top \b{1} 
\end{bmatrix}
\succeq \b{0},
\end{aligned}
\end{equation}
where $c_1$, $c_2$, and $\tau$ are constants, and the $(i,j)$-th element of $G(\b{K}_\text{tr})$ is defined as $G_{ij}(\b{K}_\text{tr}) := y_i y_j \b{K}_{ij}$. Note that $\b{K}_{ij}$ is the $(i,j)$-th element of the kernel matrix and we consider a binary classification here so the labels are $y_i \in \{-1, 1\}, \forall i$. 
The derivation of this problem is available in {\citep[Appendix B]{karimi2017summary}}.
This optimization problem is in one of the forms of semidefinite programming problem \cite{vandenberghe1996semidefinite}. See Section \ref{section_semidefinite_programming_background} for more information on semidefinite programming and how it is solved. As it is a SDP, it is a convex problem and has only one local optimum \cite{boyd2004convex}.

After learning the kernel $\b{K} \in \mathcal{K}$, it can be used for predicting the labels of unlabeled part of data, i.e., the test dataset. 
The paper \cite{lanckriet2004learning} uses kernel Support Vector Machine (SVM) \cite{vapnik1995nature} for predicting labels using the learned kernel.
Using kernelization techniques \cite{ghojogh2021reproducing}, the predictor of labels in kernel SVM becomes:
\begin{align}\label{equation_kernel_SVM_predictor}
f(\b{x}) = \sum_{i=1}^n \alpha_i\, k(\b{x}_i, \b{x}) + b,
\end{align}
where $k(.,.)$ is the kernel function \cite{ghojogh2021reproducing} which determines the elements of the kernel matrix $\b{K}_{ij} = k(\b{x}_i, \b{x}_j)$. Also, $\b{\alpha} = [\alpha_1, \dots, \alpha_n]^\top$ and $b$ is the bias. By using the learned kernel $\b{K}$ in Eq. (\ref{equation_kernel_SVM_predictor}), one can predict the labels of unlabeled test data. 
As was mentioned before, this kernel learning may have inspired the authors of MVU or SDE \cite{weinberger2004learning} to propose kernel learning for dimensionality reduction.

\section{Maximum Variance Unfolding (or Semidefinite Embedding)}\label{section_MVU}

Kernel can be learned using SDP for the sake of dimensionality reduction and manifold unfolding. In the following, we introduce kernel learning using SDP for this goal.

\subsection{Intuitions and Comparison with Kernel PCA}

As we saw in Section \ref{section_unified_framework}, most of the spectral dimensionality reduction methods can be unified as kernel PCA with different kernels. Therefore, let us learn the best kernel in dimensionality reduction for every specific dataset.

Semidefinite Embedding (SDE) \cite{weinberger2004learning,weinberger2004unsupervised,weinberger2005nonlinear,weinberger2006unsupervised}, which was renamed later to Maximum Variance Unfolding (MVU) \cite{weinberger2006introduction}, aims to find the best kernel which unfolds the manifold of data to its maximum variance. 
It learns the best kernel for manifold unfolding using semidefinite programming. 
An example manifold unfolding is shown in Fig. \ref{figure_swill_roll_unfold} where MVU gradually unfolds the nonlinear manifold to its extreme variance by iterations of semidefinite programming. 

By MVU, the embedding of data has its maximum variance in the embedding space. In this sense, the goal is similar to the goal of PCA \cite{ghojogh2019unsupervised} but there are major differences between MVU and PCA and kernel PCA. Some differences and similarities are:
\begin{itemize}
\item MVU performs embedding in the feature space or the so-called Reproducing Kernel Hilbert Space (RKHS); so does kernel PCA. However, PCA performs in the input space. 
\item MVU learns the best kernel while kernel PCA uses a ready kernel such as the Radial Basis Function (RBF) kernel. 
\item MVU is a nonlinear method. Kernel PCA transforms data to the feature space and then applies linear PCA to it. See \cite{ghojogh2021reproducing} for more explanation on the difference of these two approaches. 
\item MVU is an iterative algorithm because it solves semidefinite programming iteratively (see Section \ref{section_semidefinite_programming_background}). However, kernel PCA and PCA are almost closed form. We use the word ``almost" because solving the eigenvalue decomposition (EVD) or singular value decomposition (SVD) requires some iterations by algorithms such as the power method or Jordan's method \cite{stewart1993early} but we can see the solution of EVD or SVD as a black box. The iterative solution of semidefinite programming in MVU is much more complicated and time consuming although the task done by MVU is more promising in manifold unfolding. 
\end{itemize}

Assume we have a $d$-dimensional dataset $\mathcal{X} := \{\b{x}_i \in \mathbb{R}^d \}_{i=1}^n$. We aim to find a $p$-dimensional embedding of this dataset, denoted by $\mathcal{Y} := \{\b{y}_i \in \mathbb{R}^p \}_{i=1}^n$ where $p \leq d$ and usually $p \ll d$. 
Let $\b{X} := [\b{x}_1, \dots, \b{x}_n] \in \mathbb{R}^{d \times n}$ and $\b{Y} := [\b{y}_1, \dots, \b{y}_n] \in \mathbb{R}^{p \times n}$.
The embedding $\b{Y}$ is supposed to be the maximum variance unfolding of the manifold of data. 
As mentioned before, SDE \cite{weinberger2004learning,weinberger2004unsupervised,weinberger2005nonlinear,weinberger2006unsupervised} or MVU \cite{weinberger2006introduction} performs this task. 
MVU embeds data in the feature space or RKHS; in other words, the embedding space is the feature space (RKHS). Hence:
\begin{align}\label{equation_MVU_embedding_in_RKHS}
\b{y}_i = \b{\phi}(\b{x}_i), \quad \forall i \in \{1, \dots, n\}. 
\end{align}

\subsection{Local Isometry}

\begin{definition}[Isometric Manifolds \cite{weinberger2004unsupervised}]
Two Riemannian manifolds are isometric if there is a diffeomorphism such that the metric on one of them pulls back to the metric on the other one. In other words, isometry is is a smooth invertible mapping which locally looks like an affine transformation, i.e., a rotation and a translation. Hence, isometry preserves the local distances on the manifold.
\end{definition}
We use the notion of isometry between the data $\mathcal{X}$ and their embedding $\mathcal{Y}$. In other words, the local structure of data should be preserved in the embedding space \cite{saul2003think}. The datasets $\mathcal{X}$ and $\mathcal{Y}$ are locally isometric if they have similar rotation and translation relations between neighbor points. Let $\b{x}_j$ and $\b{x}_l$ be neighbors of $\b{x}_i$ so that they form a triangle. This triangle should also exist in the low dimensional embedding space with some rotation and translation. Hence, for isometry, we should have equal relative angles of points:
\begin{align}\label{equation_isometry_inner_product}
(\b{y}_i - \b{y}_j)^\top (\b{y}_i - \b{y}_l) = (\b{x}_i - \b{x}_j)^\top (\b{x}_i - \b{x}_l),
\end{align}
because inner product is proportional to the cosine of angle. 
A special case of Eq. (\ref{equation_isometry_inner_product}) is $l = j$:
\begin{align}\label{equation_isometry_inner_product_2}
\|\b{y}_i - \b{y}_j\|_2^2 = \|\b{x}_i - \b{x}_j\|_2^2.
\end{align}
We denote the Gram (kernel) matrices of points in the input and embedding spaces by $\mathbb{R}^{n \times n} \ni \b{G} := \b{X}^\top \b{X}$ and $\mathbb{R}^{n \times n} \ni \b{K} := \b{Y}^\top \b{Y}$, respectively. Let $\b{G}_{ij}$ and $\b{K}_{ij}$ denote the $(i,j)$-th element of $\b{G}$ and $\b{K}$, respectively. 
We have:
\begin{align*}
&\|\b{x}_i - \b{x}_j\|_2^2 = (\b{x}_i - \b{x}_j)^\top (\b{x}_i - \b{x}_j) \\
&= \b{x}_i^\top \b{x}_i + \b{x}_j^\top \b{x}_j - 2\, \b{x}_i^\top \b{x}_j = \b{G}_{ii} + \b{G}_{jj} - 2\, \b{G}_{ij}.
\end{align*}
Likewise, in the embedding space we have:
\begin{align}
&\|\b{y}_i - \b{y}_j\|_2^2 = (\b{y}_i - \b{y}_j)^\top (\b{y}_i - \b{y}_j) \nonumber \\
&= \b{y}_i^\top \b{y}_i + \b{y}_j^\top \b{y}_j - 2\, \b{y}_i^\top \b{y}_j \nonumber\\
&\overset{(\ref{equation_MVU_embedding_in_RKHS})}{=} \b{\phi}(\b{x}_i)^\top \b{\phi}(\b{x}_i) + \b{\phi}(\b{x}_j)^\top \b{\phi}(\b{x}_j) - 2\, \b{\phi}(\b{x}_i)^\top \b{\phi}(\b{x}_j) \nonumber\\
&\overset{(a)}{=} \b{K}_{ii} + \b{K}_{jj} - 2\, \b{K}_{ij}, \label{equation_distance_in_RKHS_MVU}
\end{align}
where $(a)$ is because of kernel trick \cite{ghojogh2021reproducing}. 
Note that Eq. (\ref{equation_distance_in_RKHS_MVU}) is in fact the distance of points in RKHS equipped with kernel $k$ \cite{scholkopf2001kernel,ghojogh2021reproducing}:
\begin{align*}
\|\b{y}_i - \b{y}_j\|_2^2 \overset{(\ref{equation_MVU_embedding_in_RKHS})}{=} \|\b{\phi}(\b{x}_i) - \b{\phi}(\b{x}_j)\|_k^2 =\! \b{K}_{ii} + \b{K}_{jj} -\! 2 \b{K}_{ij}.
\end{align*}
Therefore, Eq. (\ref{equation_isometry_inner_product_2}) is simplified to:
\begin{align}\label{equation_isometry_constraint}
\b{K}_{ii} + \b{K}_{jj} - 2\, \b{K}_{ij} = \b{G}_{ii} + \b{G}_{jj} - 2\, \b{G}_{ij}.
\end{align}

Some version of MVU uses $k$-Nearest Neighbors ($k$NN) in the local isometry \cite{weinberger2004learning}.
This version forms a $k$NN graph between points of training data points. Hence, we know the $k$ neighbors of every point in the dataset. Let $\tau_{ij}$ be one if $\b{x}_j$ is a neighbor of $\b{x}_i$ and zero otherwise:
\begin{align}\label{equation_MVU_tau_KNN}
\tau_{ij} := 
\left\{
    \begin{array}{ll}
        1 & \b{x}_j \in k\text{NN}(\b{x}_i) \\
        0 & \mbox{Otherwise.}
    \end{array}
\right.
\end{align}
Considering the neighboring points only to have local isometry modifies Eq. (\ref{equation_isometry_constraint}) slightly:
\begin{align}\label{equation_isometry_constraint_KNN}
\tau_{ij} (\b{K}_{ii} + \b{K}_{jj} - 2\, \b{K}_{ij}) = \tau_{ij} (\b{G}_{ii} + \b{G}_{jj} - 2\, \b{G}_{ij}).
\end{align}

\subsection{Centering}

We also want the embeddings of points to have zero mean:
\begin{align}\label{equation_mean_embedding_zero}
\sum_{i=1}^n \b{y}_i = \b{0}.
\end{align}
This removes the translational degree of freedom. 
According to Eq. (\ref{equation_MVU_embedding_in_RKHS}), centered embedding data means centered pulled data in the feature space. According to \cite{ghojogh2021reproducing}, this is equivalent to double centering the kernel matrix, resulting in:
\begin{align}\label{equation_centered_kernel}
\frac{1}{n} \sum_{i=1}^n \sum_{j=1}^n \b{K}_{ij} = \b{0} \implies \sum_{i=1}^n \sum_{j=1}^n \b{K}_{ij} = \b{0}.
\end{align}
Another justification for this is as follows. From Eq. (\ref{equation_mean_embedding_zero}), we have:
\begin{align*}
\Big|\sum_{i=1}^n \b{y}_i\Big|^2 &= \sum_{i=1}^n \sum_{j=1}^n \b{y}_i^\top \b{y}_j \overset{(\ref{equation_MVU_embedding_in_RKHS})}{=} \sum_{i=1}^n \sum_{j=1}^n \b{\phi}(\b{x}_i)^\top \b{\phi}(\b{x}_j) \nonumber \\
&\overset{(a)}{=} \sum_{i=1}^n \sum_{j=1}^n \b{K}_{ij} \overset{(\ref{equation_mean_embedding_zero})}{=} \b{0},
\end{align*}
where $(a)$ is because of the kernel trick \cite{ghojogh2021reproducing}.

\subsection{Positive Semidefiniteness}

As we saw so far, we are expressing embedding using kernels because we assumed the embedding space is the feature space. Therefore, optimization can be performed over the kernel matrix rather than the embedding points. However, a valid Mercer kernel should be symmetric and positive semidefinite \cite{ghojogh2021reproducing}. Hence, the kernel matrix should be constrained to belong to the cone of semidefinite matrices:
\begin{align}\label{equation_positive_semidefinitness_kernel_constraint}
\b{K} \in \mathbb{S}_+^n, \quad \text{or} \quad \b{K} \succeq \b{0}.
\end{align}


\subsection{Manifold Unfolding}

The goal of MVU is unfolding the manifold of data with its maximum variance, as shown in Fig. \ref{figure_swill_roll_unfold}. According to the definition of variance, MVU maximizes the following quantitative:
\begin{align}\label{equation_variance_of_embedding}
\mathcal{T}(\b{Y}) := \frac{1}{2n} \sum_{i=1}^n \sum_{j=1}^n \|\b{y}_i - \b{y}_j\|_2^2.
\end{align}

\begin{lemma}[Boundedness of Variance of Embedding \cite{weinberger2006unsupervised}]\label{lemma_boundedness_of_embedding_variance}
The value of $\mathcal{T}(\b{Y})$ is bounded. 
\end{lemma}
\begin{proof}
Suppose $\eta_{ij} = 1$ if $\b{x}_j$ is one of the $k$-nearest neighbors ($k$NN) of $\b{x}_i$; otherwise, $\eta_{ij} = 0$.
Let:
\begin{align}\label{equation_tau_small}
\tau := \max_{i,j} \Big( \eta_{ij}\, \| \b{x}_i - \b{x}_j \|_2 \Big) \overset{(a)}{<} \infty,
\end{align}
where $(a)$ is because $\| \b{x}_i - \b{x}_j \|_2 < \infty$.
Assuming that the $k$NN graph is connected, the longest path is at most $n \tau$, i.e., $\|\b{y}_i - \b{y}_j\|_2 \leq n \tau$. 
Therefore, an upper bound on Eq. (\ref{equation_variance_of_embedding}) is:
\begin{align*}
\mathcal{T}(\b{Y})\! \leq\! \frac{1}{2n} \sum_{i=1}^n \sum_{j=1}^n (n \tau)^2 = \frac{1}{2n} n^2 (n \tau)^2 = \frac{n^3 \tau^2}{2} \overset{(\ref{equation_tau_small})}{<} \infty.
\end{align*}
Q.E.D.
\end{proof}

According to Lemma \ref{lemma_boundedness_of_embedding_variance}, the variance of embedding is bounded so we can maximize it. 

The Eq. (\ref{equation_variance_of_embedding}) can be simplified as:
\begin{align}
\mathcal{T}(\b{Y}) &:= \frac{1}{2n} \sum_{i=1}^n \sum_{j=1}^n \|\b{y}_i - \b{y}_j\|_2^2 \nonumber\\
&\overset{(\ref{equation_distance_in_RKHS_MVU})}{=} \frac{1}{2n} \sum_{i=1}^n \sum_{j=1}^n (\b{K}_{ii} + \b{K}_{jj} - 2\b{K}_{ij}) \nonumber\\
&= \frac{1}{2n} \Big[ \sum_{i=1}^n \b{K}_{ii} + \sum_{j=1}^n \b{K}_{jj} - 2 \sum_{i=1}^n \sum_{j=1}^n \b{K}_{ij} \Big] \nonumber\\
&\overset{(\ref{equation_centered_kernel})}{=} \frac{1}{2n} \Big[ \sum_{i=1}^n \b{K}_{ii} + \sum_{j=1}^n \b{K}_{jj} \Big] = \frac{1}{2n} \Big[ 2 \sum_{i=1}^n \b{K}_{ii} \Big] \nonumber\\
&= \frac{1}{n} \sum_{i=1}^n \b{K}_{ii} \propto \sum_{i=1}^n \b{K}_{ii} \overset{(a)}{=} \textbf{tr}(\b{K}), \label{equation_trace_of_kernel}
\end{align}
where $\textbf{tr}(.)$ denotes the trace of matrix and $(a)$ is because trace of a matrix is equivalent to the summation of its diagonal elements. 
The Eq. (\ref{equation_trace_of_kernel}) makes sense because kernel is a measure of similarity of points \cite{ghojogh2021reproducing} so it is related to the distance of points and the variance of unfolding. 

In summary, MVU maximizes the variance of embedding, i.e. Eq. (\ref{equation_trace_of_kernel}), with constraints of Eqs. (\ref{equation_isometry_constraint}), (\ref{equation_centered_kernel}), and (\ref{equation_positive_semidefinitness_kernel_constraint}). Hence, it solves the following optimization problem:
\begin{equation}\label{equation_MVU_optimization}
\begin{aligned}
& \underset{\b{K}}{\text{maximize}}
& & \textbf{tr}(\b{K}) \\
& \text{subject to}
& & \b{K}_{ii} + \b{K}_{jj} - 2\, \b{K}_{ij} = \b{G}_{ii} + \b{G}_{jj} - 2\, \b{G}_{ij}, \\
& & & ~~~~~~~~\qquad\qquad\qquad\qquad \; \forall i,j \in \{1, \ldots, n\}, \\
& & & \sum_{i=1}^n \sum_{j=1}^n \b{K}_{ij} = 0, \\
& & & \b{K} \succeq \b{0},
\end{aligned}
\end{equation}
which is a semidefinite programming problem \cite{vandenberghe1996semidefinite}. See Section \ref{section_semidefinite_programming_background} for more information on semidefinite programming and how it is solved. It is valuable that Eq. (\ref{equation_MVU_optimization}) is a convex optimization problem so it has only one local optimum which is the global optimum \cite{boyd2004convex}.

The Eq. (\ref{equation_MVU_optimization}) is using Eq. (\ref{equation_isometry_constraint}) for local isometry. Some versions of MVU use Eq. (\ref{equation_isometry_constraint_KNN}) as the local isometry constraint:
\begin{equation}\label{equation_MVU_optimization_KNN}
\begin{aligned}
& \underset{\b{K}}{\text{maximize}}
& & \textbf{tr}(\b{K}) \\
& \text{subject to}
& & \tau_{ij} (\b{K}_{ii} + \b{K}_{jj} - 2\, \b{K}_{ij}) \\
& & & \quad\quad\quad = \tau_{ij} (\b{G}_{ii} + \b{G}_{jj} - 2\, \b{G}_{ij}), \\
& & & ~~~~~~~~\qquad\qquad\quad \; \forall i,j \in \{1, \ldots, n\}, \\
& & & \sum_{i=1}^n \sum_{j=1}^n \b{K}_{ij} = 0, \\
& & & \b{K} \succeq \b{0},
\end{aligned}
\end{equation}
Note that if we see the entire dataset as one class and set $k=n$, Eq. (\ref{equation_MVU_optimization_KNN}) becomes equivalent to Eq. (\ref{equation_MVU_optimization}).
Using $k$NN in Eq. (\ref{equation_MVU_optimization_KNN}) makes optimization slightly more efficient and faster because it does not compute kernel between all points; although, computation of $k$NN graph can be time-consuming. 

\subsection{Spectral Embedding}

After kernel is found by solving optimization (\ref{equation_MVU_optimization}), we calculate the eigenvalues and eigenvectors of the kernel \cite{ghojogh2019eigenvalue} and then the embedding of every point is obtained by Eq. (\ref{equation_embedding_eigenvector_of_kernel}) in the following lemma. The obtained embedding dataset has the maximum variance in the embedding space. 
Note that for finding the intrinsic dimensionality of manifold, denoted by $p \leq n$, we sort the eigenvalues from largest to smallest and a huge gap between two successive eigenvalues shows a good cut-off for the number of required dimensions. A scree plot can be used for example to visualize this cut-off. For better understanding of intrinsic dimensionality, see the examples in Fig. \ref{figure_intrinsic_dimensionality}.

\begin{lemma}[Embedding from Eigenvectors of Kernel Matrix \cite{ghojogh2021reproducing}]
Let $\b{v}_k = [v_{k1}, \dots, v_{kn}]^\top$ and $\delta_k$ be the $k$-th eigenvector and eigenvalue of kernel matrix, respectively. We can compute the embedding of point $\b{x}$, denoted by $\b{y}(\b{x}) = [y_1(\b{x}), \dots, y_p(\b{x})]^\top$ (where $p \leq n$) using the eigenvector of kernel as:
\begin{align}\label{equation_embedding_eigenvector_of_kernel}
y_k(\b{x}) = \sqrt{\delta_k}\, v_{ki}. 
\end{align}
\end{lemma}
\begin{proof}
See \cite{ghojogh2021reproducing} for proof. 
\end{proof}

Another way for justifying Eq. (\ref{equation_embedding_eigenvector_of_kernel}) is as follows \cite{weinberger2006unsupervised}. From the eigenvalue decomposition of kernel matrix, we have \cite{ghojogh2019eigenvalue}:
\begin{align}\label{equation_kernel_eigenvalue_decomposition}
&\b{K} = \b{V} \b{\Delta} \b{V}^\top \nonumber\\
&\implies \b{K}_{ij} = \sum_{k=1}^n \delta_k\, v_{ki}\, v_{kj}, \quad \forall i,j \in \{1, \dots, n\},
\end{align}
where the columns of $\b{V}$ are the eigenvectors and the diagonal elements of $\b{\Delta}$ are eigenvalues.
Also kernel can be stated as \cite{ghojogh2021reproducing}:
\begin{align}\label{equation_kernel_y_y_transpose}
\b{K}_{ij} = \b{\phi}(\b{x}_i)^\top \b{\phi}(\b{x}_j) \overset{(\ref{equation_MVU_embedding_in_RKHS})}{=} \b{y}_i^\top \b{y}_j.
\end{align}
Considering both Eqs. (\ref{equation_kernel_eigenvalue_decomposition}) and (\ref{equation_kernel_y_y_transpose}) gives:
\begin{align*}
&\b{y}_i^\top \b{y}_j = \sum_{k=1}^n \delta_k\, v_{ki}\, v_{kj} \overset{(a)}{=} \sum_{k=1}^n \sqrt{\delta_k}\, v_{ki}\, \sqrt{\delta_k}\, v_{kj} \\
&\implies \b{y}_i(k) = \sqrt{\delta_k}\, v_{ki},
\end{align*}
where $\b{y}_i(k)$ is the $k$-th element of $\b{y}_i$ and $(a)$ is allowed because kernel matrix is positive semidefinite (see Eq. (\ref{equation_positive_semidefinitness_kernel_constraint})) so its eigenvalues, $\delta_k$'s, are nonnegative. This equation is equal to Eq. (\ref{equation_embedding_eigenvector_of_kernel}).

\section{Supervised Maximum Variance Unfolding}\label{section_supervised_MVU}

MVU is an unsupervised manifold learning method. 
There exist several variants for supervised MVU or supervised SDE which make use of class labels. In the following, we introduce these methods. 

\subsection{Supervised MVU Using $k$NN within Classes}

One of the methods for supervised MVU or supervised SDE is \cite{zhang2005supervised,liu2005supervised} which uses $k$-Nearest Neighbors ($k$NN) within classes. 
It modifies Eq. (\ref{equation_MVU_tau_KNN}) in a way that the $k$NN graph is formed between points of every class and not amongst the entire data points. Hence, we know the $k$ neighbors of every point in each class. Let $\tau_{ij}$ be one if $\b{x}_i$ and $\b{x}_i$ belong to the same class and $\b{x}_j$ is a neighbor of $\b{x}_i$; otherwise, it is zero. 
The optimization problem of MVU is the same as Eq. (\ref{equation_MVU_optimization_KNN}) where $\tau_{ij}$ has been computed differently. 

\subsection{Supervised MVU by Class-wise Unfolding}

The method named SMVU1 proposed in \cite{wei2016developments} is a supervised MVU method which unfolds manifold class-wise. 
Let $G_c$ denote the set of points belonging to the class $c$ and $\bar{\b{x}}_c$ be the mean of class $c$.
For the class $c$, we define the representative $\b{x}_c$ to be the closest point of class to the mean of class:
\begin{align}
\b{x}_c := \underset{\b{x}_i \in G_c}{\min} \|\b{x}_i - \bar{\b{x}}_c\|_2^2.
\end{align}
Note that, rather than the above definition, one can define $\b{x}_c$ to be the medoid of class which is the closest point to all points of the class; however, computation of medoid can be more time-consuming. 

Assume we have $C$ classes denoted by $\{c_1, \dots, c_C\}$. 
Recall that local isometry yielded $\b{K}_{ii} + \b{K}_{jj} - 2\, \b{K}_{ij} = \|\b{x}_i - \b{x}_j\|_2^2$.
For class-wise local isometry, we can have:
\begin{align}
\b{K}_{c_i c_i} + \b{K}_{c_j c_j} - 2\, \b{K}_{c_i c_j} = \alpha^2 \|\b{x}_{c_i} - \b{x}_{c_j}\|_2^2,
\end{align}
$\forall i,j \in \{1, \dots, C\}$, 
where $\alpha > 1$ is a hyperparameter \cite{wei2016developments} and $\b{K}_{c_i c_j} := \b{\phi}(\b{x}_{c_i})^\top \b{\phi}(\b{x}_{c_j})$ is the kernel between representatives of classes $i$ and $j$.
The paper \cite{wei2016developments} considers the following pairs of classes:
\begin{align*}
&\b{K}_{c_i c_i} + \b{K}_{c_{i+1} c_{i+1}} - 2\, \b{K}_{c_i c_{i+1}} = \alpha^2 \|\b{x}_{c_i} - \b{x}_{c_j}\|_2^2, \\
&\b{K}_{c_i c_i} + \b{K}_{c_C c_C} - 2\, \b{K}_{c_i c_C} = \alpha^2 \|\b{x}_{c_i} - \b{x}_C\|_2^2,
\end{align*}
$\forall i,j \in \{1, \dots, C\}$.
We define:
\begin{align}
& \Gamma_c = \frac{1}{2} \sum_{i=1}^{n_c} \sum_{j=1}^{n_c} (\b{K}_{ii} + \b{K}_{jj} - 2\b{K}_{ij}),\; \forall \b{x}_i, \b{x}_j \in G_c, \\
& \Gamma := \sum_{c=1}^C \frac{\Gamma_c}{n_c}.
\end{align}
where $n_c$ denotes the number of points in class $c$. We maximize this term to maximize the variance of unfolding for each class. The SDP optimization for kernel learning is \cite{wei2016developments}:
\begin{equation}\label{equation_supervised_MVU_optimization_2}
\begin{aligned}
& \underset{\b{K}}{\text{maximize}} \quad \sum_{c=1}^C \frac{\Gamma_c}{n_c} \\
& \text{subject to} \\
& \tau_{ij} (\b{K}_{ii} + \b{K}_{jj} - 2\, \b{K}_{ij}) = \tau_{ij} (\b{G}_{ii} + \b{G}_{jj} - 2\, \b{G}_{ij}), \\
& ~~~~~~~~\qquad\qquad\qquad\qquad\qquad \; \forall i,j \in \{1, \ldots, n\}, \\
& \b{K}_{c_i c_i} + \b{K}_{c_{i+1} c_{i+1}} - 2\, \b{K}_{c_i c_{i+1}} = \alpha^2 \|\b{x}_{c_i} - \b{x}_{c_j}\|_2^2, \\
& \b{K}_{c_i c_i} + \b{K}_{c_C c_C} - 2\, \b{K}_{c_i c_C} = \alpha^2 \|\b{x}_{c_i} - \b{x}_C\|_2^2, \\
& ~~~~~~~~\qquad\qquad\qquad\qquad\qquad \; \forall i,j \in \{1, \dots, C\}, \\
& \sum_{i=1}^n \sum_{j=1}^n \b{K}_{ij} = 0, \\
& \b{K} \succeq \b{0},
\end{aligned}
\end{equation}
where $\tau_{ij}$ is the same as defined before in Eq. (\ref{equation_MVU_tau_KNN}). 

\subsection{Supervised MVU by Fisher Criterion}

The method named SMVU2 proposed in \cite{wei2016developments} is a supervised MVU method which unfolds manifold by Fisher criterion; hence, we name this method Fisher-MVU. Fisher criterion maximizes the between-class scatter and minimizes the within-class scatter. The within-class scatter in the embedding space (or feature space equipped with kernel $k$) is:
\begin{align}
\sigma_W &:= \sum_{c=1}^C \frac{1}{n_c} \sum_{\b{x}_i \in G_c} \|\b{\phi}(\b{x}_i) - \b{\phi}(\b{x}_c)\|_k^2 \nonumber \\
&\overset{(a)}{=} \sum_{c=1}^C \frac{1}{n_c} \sum_{\b{x}_i \in G_c} (\b{K}_{ii} + \b{K}_{cc} - 2\b{K}_{ci}),
\end{align}
where $\b{K}_{ii} := \b{\phi}(\b{x}_i)^\top \b{\phi}(\b{x}_i)$, $\b{K}_{cc} := \b{\phi}(\b{x}_c)^\top \b{\phi}(\b{x}_c)$, and $\b{K}_{ci} := \b{\phi}(\b{x}_c)^\top \b{\phi}(\b{x}_i)$, and $(a)$ is because of distance in the feature space \cite{scholkopf2001kernel,ghojogh2021reproducing}. 
The between-class scatter in the embedding space (or feature space equipped with kernel $k$) is:
\begin{align}
\sigma_B &:= \sum_{c_i=1}^C \sum_{c_j=1}^C \|\b{\phi}(\b{x}_{c_i}) - \b{\phi}(\b{x}_{c_j})\|_k^2 \nonumber \\
&\overset{(a)}{=} \sum_{c_i=1}^C \sum_{c_j=1}^C (\b{K}_{c_i c_i} + \b{K}_{c_j c_j} - 2\b{K}_{c_i c_j}).
\end{align}
where $(a)$ is because of distance in the feature space \cite{scholkopf2001kernel,ghojogh2021reproducing}. 
Note that the between-class scatter has another form which uses the mean of classes; however, as the mean of a class is not necessarily one of the points, its embedding (in the feature space) is not available. Therefore, the scatter of all points from classes is used for computation of between-class scatter.
One of the forms for Fisher criterion is \cite{fukunaga1990introduction}:
\begin{align}
\Gamma = C \times (\sigma_B - \sigma_W),
\end{align}
which should be maximized. 
The SDP optimization for kernel learning is \cite{wei2016developments}:
\begin{equation}\label{equation_supervised_MVU_optimization_3}
\begin{aligned}
& \underset{\b{K}}{\text{maximize}}
& & C (\sigma_B - \sigma_W) \\
& \text{subject to}
& & \tau_{ij} (\b{K}_{ii} + \b{K}_{jj} - 2\, \b{K}_{ij}) \\
& & & \quad\quad\quad = \tau_{ij} (\b{G}_{ii} + \b{G}_{jj} - 2\, \b{G}_{ij}), \\
& & & ~~~~~~~~\qquad\qquad\quad \; \forall i,j \in \{1, \ldots, n\}, \\
& & & \sum_{i=1}^n \sum_{j=1}^n \b{K}_{ij} = 0, \\
& & & \b{K} \succeq \b{0}.
\end{aligned}
\end{equation}

\subsection{Supervised MVU by Colored MVU}

One of the supervised approches for MVU is colored MVU \cite{song2007colored} which uses some side information such as labels. In its formulation, it uses the Hilbert-Schmidt Independence Criterion (HSIC) \cite{gretton2005measuring} between embedded data points $\{\b{y}_i = \b{\phi}(\b{x}_i) \}_{i=1}^n$ and labels $\{l_i \}_{i=1}^n$:
\begin{align}
\text{HSIC} &:= \frac{1}{(n-1)^2}\, \textbf{tr}(\b{K}_l\b{H}\b{K}\b{H}) \nonumber \\
&\overset{(a)}{=} \frac{1}{(n-1)^2}\,  \textbf{tr}(\b{H}\b{K}\b{H}\b{K}_l), \label{equation_HSIC}
\end{align}
where $(a)$ is because of cyclic property of trace, $\b{K} \in \mathbb{R}^{n \times n}$ and $\b{K}_l \in \mathbb{R}^{n \times n}$ are kernel matrices over embedded points and labels, respectively, and $\mathbb{R}^{n \times n} \ni \b{H} := \b{I} - (1/n) \b{1}\b{1}^\top$ is the centering matrix. 
HSIC is a measure of dependence between two random variables \cite{ghojogh2021reproducing}. 
Colored MVU maximizes the HSIC between the embedded points and the labels; in other words, the dependence of embedding and labels is maximized to be supervised. Therefore, colored MVU maximizes Eq. (\ref{equation_HSIC}), i.e. $\textbf{tr}(\b{H}\b{K}\b{H}\b{K}_l)$ or $\textbf{tr}(\b{K}\b{H}\b{K}_l\b{H})$, rather than maximizing $\textbf{tr}(\b{K})$ which is done in MVU \cite{song2007colored}:
\begin{equation}\label{equation_colored_MVU_optimization}
\begin{aligned}
& \underset{\b{K}}{\text{maximize}}
& & \textbf{tr}(\b{H}\b{K}\b{H}\b{K}_l) \\
& \text{subject to}
& & \tau_{ij} (\b{K}_{ii} + \b{K}_{jj} - 2\, \b{K}_{ij}) \\
& & & \quad\quad\quad = \tau_{ij} (\b{G}_{ii} + \b{G}_{jj} - 2\, \b{G}_{ij}), \\
& & & ~~~~~~~~\qquad\qquad\quad \; \forall i,j \in \{1, \ldots, n\}, \\
& & & \sum_{i=1}^n \sum_{j=1}^n \b{K}_{ij} = 0, \\
& & & \b{K} \succeq \b{0}.
\end{aligned}
\end{equation}
Note that as kernel is a soft measure of similarity \cite{ghojogh2021reproducing}, the labels or side information $\{l_i \}_{i=1}^n$ can be soft labels (e.g., regression labels) or hard labels (e.g., classification labels). In case of hard labels, one of the best choices for the $\b{K}_l$ is delta kernel \cite{barshan2011supervised} where the $(i,j)$-th element of kernel is:
\begin{align}
\b{K}_l(i,j) = \delta_{\b{l}_i, \b{l}_j} := 
\left\{
\begin{array}{ll}
  1 & \text{if } \b{l}_i = \b{l}_j,\\
  0 & \text{if } \b{l}_i \neq \b{l}_j,
\end{array}
\right.
\end{align}
where $\delta_{\b{l}_i, \b{l}_j}$ is the Kronecker delta which is one if the $\b{x}_i$ and $\b{x}_j$ belong to the same class.

\section{Out-of-sample Extension of MVU}\label{section_MVU_outofsample}

There are several approaches for out-of-sample extension of MVU. In the following, we introduce some of these approaches which are approaches using eigenfunctions \cite{chin2008out} and kernel mapping \cite{gisbrecht2015parametric}. There exist some other methods for out-of-sample extension such as \cite{bunte2012general} which we do not cover here. 

\subsection{Out-of-sample Extension Using Eigenfunctions}

One of the methods for out-of-sample extension of MVU is using eigenfunctions \cite{chin2008out}. 
Recall Eq. (\ref{equation_relation_eigenfunction_eigenvector_x}) which relates the eigenvectors of kernel function and eigenfunctions of kernel operator. 
According to \cite{schwaighofer2005learning}, this equation can be slightly modified to:
\begin{align}\label{equation_MVU_outofsample_eigenfunction}
&f_k(\b{x}) = \sum_{i=1}^n b_{ki}\, r(\b{x}_i, \b{x}) = \sum_{i=1}^n b_{ki}\, r(\b{x}, \b{x}_i), 
\end{align}
where $r(.,.)$ is an auxiliary smoothing kernel such as the RBF kernel, and $b_{ki}$ is the $i$-th element of:
\begin{align}\label{equation_MVU_outofsample_b}
\mathbb{R}^n \ni \b{b}_k := (\b{R} + \eta \b{I})^{-1} \b{v}_k,
\end{align}
where $\b{v}_k$ is the $k$-th eigenvector of kernel matrix $\b{K} \in \mathbb{R}^{n \times n}$ over training data, and $\b{R} \in \mathbb{R}^{n \times n}$ is the smoothing kernel matrix on the $n$ training data points using kernel $r(.,.)$ and $\eta > 0$ is a regularization parameter for stable inverse of matrix. 

According to the Mercer's theorem \cite{ghojogh2021reproducing}, the kernel can be written as \cite{chin2008out}:
\begin{align}
&k(\b{x}, \b{x}_i) = \sum_{j=1}^n \lambda_j \psi_j(\b{x}) \psi_j(\b{x}_i) \nonumber \\
&~~~~~~~ \overset{(a)}{=}
\b{r}(\b{x})^\top (\b{R} + \eta \b{I})^{-1} \b{K} (\b{R} + \eta \b{I})^{-1} \b{r}(\b{x}_i) \label{equation_MVU_outofsample_k}
\end{align}
where $(a)$ is because of Eqs. (\ref{equation_MVU_outofsample_eigenfunction}) and (\ref{equation_MVU_outofsample_b}), $\b{r}(\b{x}) := [r(\b{x}_1, \b{x}), \dots, r(\b{x}_n, \b{x})]^\top \in \mathbb{R}^n$, $\{\psi_j\}_{j=1}^n$ and $\{\lambda_j\}_{j=1}^n$ are the eigenfunctions and eigenvalues of the kernel operator $k$, respectively. 
According to Eq. (\ref{equation_embedding_eigenfunction}), we have:
\begin{align}\label{equation_MVU_outofsample_y_embedding}
y_k(\b{x}) &= \frac{1}{\sqrt{\delta_k}} \sum_{i=1}^n v_{ki}\, k(\b{x}, \b{x}_i) \overset{(\ref{equation_MVU_outofsample_k})}{=} \b{p}_k\, \b{r}(\b{x}),
\end{align}
where $\delta_k$ is the $k$-th eigenvalue of kernel matrix $\b{K}$ and \cite{chin2008out}:
\begin{align}
\mathbb{R}^{1 \times n} \ni \b{p}_k := \delta_k^{-1/2} \b{v}_k^\top \b{R} (\b{R} + \eta \b{I})^{-1} \b{K} (\b{R} + \eta \b{I})^{-1}.
\end{align}
Therefore, for a point $\b{x}$ which can be an out-of-sample point, the $k$-th dimension of embedding is calculated using Eq. (\ref{equation_MVU_outofsample_y_embedding}) and considering the top $p \leq n$ dimensions of embedding gives $\mathbb{R}^p \ni \b{y}(\b{x}) = [y_1(\b{x}), \dots, y_p(\b{x})]^\top$. 

\subsection{Out-of-sample Extension Using Kernel Mapping}

There is a kernel mapping method \cite{gisbrecht2015parametric} to embed the out-of-sample data in MVU.
We introduce this method here.
We define a map which maps any data point as $\b{x} \mapsto \b{y}(\b{x})$, where:
\begin{align}\label{equation_kernel_tSNE_map}
\mathbb{R}^p \ni \b{y}(\b{x}) := \sum_{j=1}^n \b{\alpha}_j\, \frac{k(\b{x}, \b{x}_j)}{\sum_{\ell=1}^n k(\b{x}, \b{x}_{\ell})},
\end{align}
and $\b{\alpha}_j \in \mathbb{R}^p$, and $\b{x}_j$ and $\b{x}_{\ell}$ denote the $j$-th and $\ell$-th training data points, respectively.
The $k(\b{x}, \b{x}_j)$ is a kernel such as the Gaussian kernel:
\begin{align}
k(\b{x}, \b{x}_j) = \exp(\frac{-||\b{x} - \b{x}_j||_2^2}{2\, \sigma_j^2}),
\end{align}
where $\sigma_j$ is calculated as \cite{gisbrecht2015parametric}:
\begin{align}
\sigma_j := \gamma \times \min_{i}(||\b{x}_j - \b{x}_i||_2),
\end{align}
where $\gamma$ is a small positive number.

Assume we have already embedded the training data points using MVU; therefore, the set $\{\b{y}_i\}_{i=1}^n$ is available.
If we map the training data points, we want to minimize the following least-squares cost function in order to get $\b{y}(\b{x}_i)$ close to $\b{y}_i$ for the $i$-th training point:
\begin{equation}
\begin{aligned}
& \underset{\b{\alpha}_j\text{'s}}{\text{minimize}}
& & \sum_{i=1}^n ||\b{y}_i - \b{y}(\b{x}_i)||_2^2,
\end{aligned}
\end{equation}
where the summation is over the training data points.
We can write this cost function in matrix form as below:
\begin{equation}\label{equation_kernel_tSNE_leastSquares}
\begin{aligned}
& \underset{\b{A}}{\text{minimize}}
& & ||\b{Y} - \b{K}''\b{A}||_F^2,
\end{aligned}
\end{equation}
where $\mathbb{R}^{n \times p} \ni \b{Y} := [\b{y}_1, \dots, \b{y}_n]^\top$ and $\mathbb{R}^{n \times p} \ni \b{A} := [\b{\alpha}_1, \dots, \b{\alpha}_n]^\top$. 
The $\b{K}'' \in \mathbb{R}^{n \times n}$ is the kernel matrix whose $(i,j)$-th element is defined to be:
\begin{align}
\b{K}''(i,j) := \frac{k(\b{x}_i, \b{x}_j)}{\sum_{\ell=1}^n k(\b{x}_i, \b{x}_{\ell})}.
\end{align}
The Eq. (\ref{equation_kernel_tSNE_leastSquares}) is always non-negative; thus, its smallest value is zero.
Therefore, the solution to this equation is:
\begin{align}
\b{Y} - \b{K}''\b{A} = \b{0} &\implies \b{Y} = \b{K}''\b{A} \nonumber \\
&\overset{(a)}{\implies} \b{A} = \b{K}''^{\dagger}\, \b{Y}, \label{equation_kernel_tSNE_A_matrix}
\end{align}
where $\b{K}''^{\dagger}$ is the pseudo-inverse of $\b{K}''$:
\begin{align}
\b{K}''^{\dagger} = (\b{K}''^\top \b{K}'')^{-1} \b{K}''^\top,
\end{align}
and $(a)$ is because $\b{K}''^{\dagger}\,\b{K}'' = \b{I}$.

Finally, the mapping of Eq. (\ref{equation_kernel_tSNE_map}) for the $n_t$ out-of-sample data points is:
\begin{align}\label{equation_kernel_tSNE_outOfSample_Y}
\b{Y}_t = \b{K}''_t\,\b{A}, 
\end{align}
where the $(i,j)$-th element of the out-of-sample kernel matrix $\b{K}''_t \in \mathbb{R}^{n_t \times n}$ is:
\begin{align}
\b{K}''_t(i,j) := \frac{k(\b{x}_i^{(t)}, \b{x}_j)}{\sum_{\ell=1}^n k(\b{x}_i^{(t)}, \b{x}_{\ell})},
\end{align}
where $\b{x}_i^{(t)}$ is the $i$-th out-of-sample data point, and $\b{x}_j$ and $\b{x}_{\ell}$ are the $j$-th and $\ell$-th training data points.

\section{Other Variants of Maximum Variance Unfolding}\label{section_MVU_other_variants}

\subsection{Action Respecting Embedding}

Most of dimensionality reduction methods, including MVU, do not care about the order of points. However, some data may consist of temporal information; for example, the frames of a video where the order of images matters. In this case, a variant of MVU is required which considers the temporal information when unfolding the manifold of data. Action Respecting Embedding (ARE) \cite{bowling2005action} is an MVU variant caring about order of points. It has various applications in reinforcement learning and robotics \cite{bowling2007subjective}. 
ARE has the same constraints in the optimization of MVU plus one additional constraint about temporal information. 
Assume two actions (e.g., rotation or transformation of image or a combination of rotation and transformation), denoted by $a_i$ and $a_j$, are applied on data points $\b{x}_i$ and $\b{x}_j$ to result in $\b{x}_{i+1}$ and $\b{x}_{j+1}$, respectively:
\begin{align*}
& \b{x}_i \overset{a_i}{\longrightarrow} \b{x}_{i+1}, \\
& \b{x}_j \overset{a_j}{\longrightarrow} \b{x}_{j+1}.
\end{align*}
Consider the same actions in the embedding space of MVU, i.e. the feature space:
\begin{align*}
& \b{\phi}(\b{x}_i) \overset{a_i}{\longrightarrow} \b{\phi}(\b{x}_{i+1}), \\
& \b{\phi}(\b{x}_j) \overset{a_j}{\longrightarrow} \b{\phi}(\b{x}_{j+1}).
\end{align*}
If the two actions $a_i$ and $a_j$ are equal, the distances of embedded points should remain the same before and after the action:
\begin{align}
&a_i = a_j \implies \|\b{\phi}(\b{x}_i) - \b{\phi}(\b{x}_j)\|_k^2 \nonumber \\
&~~~~~~~~~~~~~~~~~~~~~~~~~ = \|\b{\phi}(\b{x}_{i+1}) - \b{\phi}(\b{x}_{j+1})\|_k^2 \nonumber \\
&\overset{(a)}{\implies} \b{K}_{ii} + \b{K}_{jj} - 2\b{K}_{ij} \nonumber \\
&~~~~~~~~~~~~~~~~ = \b{K}_{{i+1},{i+1}} + \b{K}_{{j+1},{j+1}} - 2\b{K}_{{i+1},{j+1}},
\end{align}
where $(a)$ is because of distance in the feature space \cite{scholkopf2001kernel,ghojogh2021reproducing}.
Finally, given actions $\{a_i\}_{i=1}^n$, the optimization in ARE is:
\begin{equation}\label{}
\begin{aligned}
& \underset{\b{K}}{\text{maximize}} \quad \textbf{tr}(\b{K}) \\
& \text{subject to} \\
& \b{K}_{ii} + \b{K}_{jj} - 2\, \b{K}_{ij} = \b{G}_{ii} + \b{G}_{jj} - 2\, \b{G}_{ij}, \\
& ~~~~~~~~\qquad\qquad\qquad\qquad\qquad \; \forall i,j \in \{1, \ldots, n \}, \\
& \b{K}_{ii} \!+\! \b{K}_{jj} \!- 2\b{K}_{ij} = \b{K}_{{i+1},{i+1}} \!+\! \b{K}_{{j+1},{j+1}} \!- 2\b{K}_{{i+1},{j+1}}, \\
& ~~~~~~~~\qquad\qquad\qquad\qquad\qquad \; \forall i,j : a_i = a_j, \\
& \sum_{i=1}^n \sum_{j=1}^n \b{K}_{ij} = 0, \\
& \b{K} \succeq \b{0},
\end{aligned}
\end{equation}
which is a SDP problem. The solution of this problem gives a kernel for unfolding the manifold of data where the temporal information of actions is taken into account. After finding the kernel from optimization, the unfolded embedding is calculated by Eq. (\ref{equation_embedding_eigenvector_of_kernel}).

\subsection{Relaxed Maximum Variance Unfolding}

Two problems exist in MVU which are addressed in relaxed MVU \cite{hou2008relaxed}. In the following, we explain these problems and how relaxed MVU resolves them.

\subsubsection{Short Circuits in $k$NN Graph}

\begin{figure}[!t]
\centering
\includegraphics[width=3.2in]{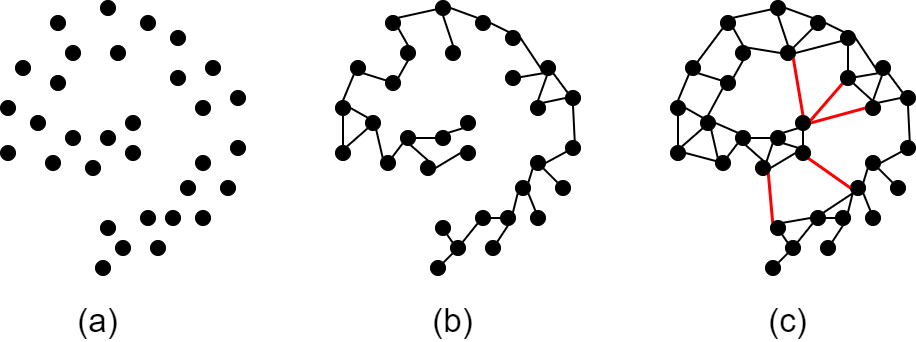}
\caption{Visualizing the possible problem of short circuit: (a) data points lying on a Swiss roll manifold, (b) Correct $k$NN graph for the manifold, (c) and incorrect $k$NN graph having short circuits (shown by red edges) by some larger value for $k$.}
\label{figure_short_circuit}
\end{figure}

For some $k$ values, some short circuits may happen in $k$NN graph. For example, see Fig. \ref{figure_short_circuit}. These short circuits result in incorrect unfolding of manifold. 
Let $k\text{NN}(\b{x}_i)$ denote the set of $k$ nearest neighbor points of $\b{x}_i$. Consider a $k$NN of data where $v(i,j)$ is an edge between $\b{x}_i$ and $\b{x}_j$ in this graph. We define the deviation of an edge $v(i,j)$ as \cite{hou2008relaxed}:
\begin{align}
d(v(i,j)) := &\frac{1}{| k\text{NN}(\b{x}_i) \cup k\text{NN}(\b{x}_j) |} \times \nonumber \\
&\sum_{ \b{x}_l \in k\text{NN}(\b{x}_i) \cup k\text{NN}(\b{x}_j) } \|\b{x}_l - \b{x}_{ij}^\text{(mid)}\|_2^2,
\end{align}
where $\b{x}_{ij}^\text{(mid)} := (\b{x}_i + \b{x}_j) / 2$ and $|.|$ is the size of set. This deviation is related to the density of points; that is, the lower the density, the larger the deviation. We sort the deviation of points from smallest to largest and can discard the points from $k$NN which are larger than a threshold. A scree plot can be used for finding the suitable threshold. 

\subsubsection{Rescaling Local Distances}

In some cases, the mapping is conformal but not isometric. Conformal maps are locally isometric but up to a scale. Let $s(\b{x}_i)$ denote the average distance of $\b{x}_i$ to its $k$ nearest neighbors. Assuming that the original sampling in the input space is uniform, the conformal scaling factor between points $\b{x}_i$ and $\b{x}_j$ is $(s(\b{x}_i) s(\b{x}_j))^{1/2}$ \cite{de2003global}. Relaxed MVU shows to be robust to violation of this assumption \cite{hou2008relaxed}. Relaxed MVU scales the distances of points in the local isometry constraint:
\begin{align}
&\tau_{ij} (\b{K}_{ii} + \b{K}_{jj} - 2\, \b{K}_{ij}) \nonumber \\
&\quad\quad = \tau_{ij} \|(s(\b{x}_i) s(\b{x}_j))^{1/2} (\b{x}_i - \b{x}_j) \|_2^2.
\end{align}

\subsection{Landmark Maximum Variance Unfolding for Big Data}

As explained in Section \ref{section_semidefinite_programming_background}, MVU uses interior-point method which is slow especially for big data. Landmark SDE or Landmark MVU \cite{weinberger2005nonlinear} uses randomly selected landmarks from data points. Let $n$ and $m \ll n$ denote the total number of points and the number of landmarks, respectively, and $\{\b{\mu}_\alpha\}_{\alpha=1}^m \subset \{\b{x}_i\}_{i=1}^n$ be the landmarks.
Every point can be reconstructed linearly by the landmarks:
\begin{align}
\widehat{\b{x}}_i = \sum_{\alpha=1}^m q_{i\alpha}\, \b{\mu}_\alpha,
\end{align}
where $q_{i\alpha}, \forall i, \alpha$ are the reconstruction weights. 
Inspired by Locally Linear embedding (LLE) \cite{ghojogh2020locally}, every embedded point should also be reconstructed from the embedded landmarks, denoted by $\{\b{\ell}_\alpha\}_{\alpha=1}^m = \{\b{\phi}(\b{\mu}_\alpha)\}_{\alpha=1}^m$, with the same reconstruction weights:
\begin{align}\label{equation_landmark_MVU_y_hat}
\widehat{\b{y}}_i = \sum_{\alpha=1}^m q_{i\alpha}\, \b{\ell}_\alpha.
\end{align}
Kernel can be stated as inner product of points in the feature space \cite{ghojogh2021reproducing}:
\begin{align}
&\b{K}_{ij} = \b{\phi}(\b{x}_i)^\top \b{\phi}(\b{x}_j) \overset{(\ref{equation_MVU_embedding_in_RKHS})}{=} \b{y}_i^\top \b{y}_j \implies \b{K}_{ij} \approx \widehat{\b{y}}_i^\top \widehat{\b{y}}_j. \nonumber \\
&\overset{(\ref{equation_landmark_MVU_y_hat})}{\implies} \b{K}_{ij} \approx q_{i \alpha} \b{\ell}_\alpha^\top \b{\ell}_\beta q_{i \beta} \implies \b{K} = \b{V} \b{L} \b{V}^\top, \label{equation_landmark_MVU_K_decompose}
\end{align}
where $\b{Q}_{i\alpha} = q_{i \alpha}$, $\b{Q} \in \mathbb{R}^{n \times m}$, and $\b{L} = \b{\ell}_\alpha^\top \b{\ell}_\beta \in \mathbb{R}^{m \times m}$.
This decomposition of $\b{L}$ shows that $\b{L} \succeq \b{0}$ which will be used as one of the constraints in optimization.
Now, consider linear reconstruction of points by all points in the training set:
\begin{align}
\min_{w_{ij}}\quad \sum_{i=1}^n \|\b{x}_i - \sum_{j=1}^n w_{ij} \b{x}_j\|_2^2,
\end{align}
where $w_{ij}$'s are the reconstruction weights. 
This optimization also exists in LLE and can be restated as \cite{ghojogh2020locally}:
\begin{align}\label{equation_landmarkMVU_min_reconstruction_error_all_points}
\min_{\b{M}}\quad \sum_{i=1}^n \sum_{j=1}^n \b{M}_{ij} \b{x}_i \b{x}_j,
\end{align}
where $\mathbb{R}^{n \times n} \ni \b{M} := (\b{I} - \b{W})^\top (\b{I} - \b{W})$ where $\b{W}_{ij} = w_{ij}$.
The matrix $\b{M}$ can be found by solving the least squares problem in Eq. (\ref{equation_landmarkMVU_min_reconstruction_error_all_points}). 
Considering the $m \ll n$ landmarks, we can decompose the matrix $\b{M}$ as:
\begin{align}\label{equation_landmarkMVU_M_decompose}
\b{M} \approx 
\left[
\begin{array}{c|c}
\b{M}^{ll} \in \mathbb{R}^{m \times m} & \b{M}^{lu} \in \mathbb{R}^{m \times (n-m)} \\
\hline
\b{M}^{ul} \in \mathbb{R}^{(n-m) \times m} & \b{M}^{uu} \in \mathbb{R}^{(n-m) \times (n-m)}
\end{array}
\right].
\end{align}
As the reconstruction weights $q_{i\alpha}, \forall i, \alpha$ can be seen as a subset of reconstruction weights $w_{ij}$, we can write $\b{Q}$ using the parts of matrix $\b{M}$ \cite{weinberger2005nonlinear}:
\begin{align}\label{equation_landmark_MVU_V}
\mathbb{R}^{n \times m} \ni \b{Q} = 
\begin{bmatrix}
\b{I}_{m \times m}\\
(\b{M}^{uu})^{-1} \b{M}^{ul}
\end{bmatrix}.
\end{align}
Note that this usage of a small part of matrix (because $m \ll n$) is inspired by the Nystr\"om method \cite{ghojogh2021reproducing}.

Landmark MVU solves the following SDP optimization problem \cite{weinberger2005nonlinear}:
\begin{equation}\label{equation_landmark_MVU_optimization}
\begin{aligned}
& \underset{\b{L}}{\text{maximize}}
& & \textbf{tr}(\b{Q}\b{L}\b{Q}^\top) \\
& \text{subject to}
& & \tau_{ij} ((\b{Q}\b{L}\b{Q}^\top)_{ii} + (\b{Q}\b{L}\b{Q}^\top)_{jj} - 2\, (\b{Q}\b{L}\b{Q}^\top)_{ij}) \\
& & & \quad\quad\quad \leq \tau_{ij} (\b{G}_{ii} + \b{G}_{jj} - 2\, \b{G}_{ij}), \\
& & & ~~~~~~~~\qquad\qquad\quad \; \forall i,j \in \{1, \ldots, n\}, \\
& & & \sum_{i=1}^n \sum_{j=1}^n (\b{Q}\b{L}\b{Q}^\top)_{ij} = 0, \\
& & & \b{L} \succeq \b{0},
\end{aligned}
\end{equation}
where the optimization variable is the small matrix $\b{L} \in \mathbb{R}^{m \times m}$ rather than the large matrix $\b{K} \in \mathbb{R}^{n \times n}$; hence, a big problem is reduced to an efficient small one. Note that the paper \cite{weinberger2005nonlinear} has converted the equality of local isometry constraint to inequality. This is okay because it has made the constraint more restricted and harder than the equality constraint. After solving the optimization problem to find the optimal $\b{L}$, we use Eqs. (\ref{equation_landmark_MVU_V}) and (\ref{equation_landmark_MVU_K_decompose}) to calculate $\b{K}$ and then embedding is calculated using Eq. (\ref{equation_embedding_eigenvector_of_kernel}).

\subsection{Other Improvements over Maximum Variance Unfolding and Kernel Learning}

An existing book chapter on MVU is {\citep[Chapter 9]{wang2012geometric}}.
There have been other improvements over MVU. 
There exist some other improvements over MVU and kernel learning by SDP whose details we do not cover in this paper. We list these improvements here.
An application of MVU in nonlinear process monitoring can be found in \cite{liu2014nonlinear}. 
SDP has also been used for kernel matrix completion \cite{graepel2002kernel} and low-rank kernel learning \cite{kulis2006learning}. 
Maximum covariance unfolding \cite{mahadevan2011maximum} has been proposed for bimodal manifold unfolding. 
We can also interpret MVU as a regularized shortest path problem on the graph of data \cite{paprotny2012connection}; hence, it can be related to the Isomap algorithm \cite{ghojogh2020multidimensional}.

\section{Conclusion}\label{section_conclusion}

In this paper, we first explained how the spectral dimensionality reduction methods can be unified as kernel PCA with different kernels using eigenfunction learning and kernel construction by distance matrices. Then, we said as the spectral methods are unified as kernel PCA, let us learn the best kernel for unfolding the manifold of data. We introduced MVU and its variants as some methods for learning the best kernel for manifold unfolding using SDP optimization problems.

\section*{Acknowledgement}

About the background on semidefinite programming, great videos of Convex Optimization I and II by Prof. Stephen Boyd exist in the channel of Stanford University on YouTube. 
Videos on unifying spectral methods as well as MVU and ARE exist by Prof. Ali Ghodsi at University of Waterloo available on YouTube.


\bibliography{References}
\bibliographystyle{icml2016}

\end{document}